\documentclass[letterpaper]{article} 
\usepackage{aaai2026}  
\usepackage{times}  
\usepackage{helvet}  
\usepackage{courier}  
\usepackage[hyphens]{url}  
\usepackage{graphicx} 
\usepackage{xcolor}
\definecolor{budgetcolor}{RGB}{0,100,200}
\definecolor{roicolor}{RGB}{200,0,50}
\usepackage{natbib}  
\usepackage{caption} 
\usepackage{newtxmath}
\usepackage{subcaption} 
\usepackage{algorithm}
\usepackage{algpseudocode}
\usepackage{ntheorem}
\newtheorem{theorem}{Theorem}
\newtheorem{lemma}{Lemma}
\newtheorem{definition}{Definition}
\newtheorem*{proof}{Proof}
\usepackage{multirow, booktabs}
\usepackage{newfloat}

\usepackage{listings}

\DeclareCaptionStyle{ruled}{labelfont=normalfont,labelsep=colon,strut=off} 
\floatstyle{ruled}
\newfloat{listing}{tb}{lst}{}
\floatname{listing}{Listing}
\usepackage{amsmath}
\nocopyright 

\title{HALO: Hindsight-Augmented Learning for Online Auto-Bidding}
\author{
    Pusen Dong,
    Chenglong Cao,
    Xinyu Zhou,
    Jirong You,
    Linhe Xu,
    Feifan Xu,
    Shuo Yuan\thanks{Corresponding author}
}
\affiliations{
    Shopee Company, China\\

    \{pusen.dong, chenglong.cao, xinyu.zhou, jirong.you, linhe.xu, feifan.xu, cody.tan\}@shopee.com
}

\usepackage{bibentry}

\begin{document}

\maketitle

\begin{abstract}
Digital advertising platforms operate millisecond-level auctions through Real-Time Bidding (RTB) systems, where advertisers compete for ad impressions through algorithmic bids. This dynamic mechanism enables precise audience targeting but introduces profound operational complexity due to advertiser heterogeneity: budgets and ROI targets span orders of magnitude across advertisers, from individual merchants to multinational brands.  This diversity creates a demanding adaptation landscape for Multi-Constraint Bidding (MCB). Traditional auto-bidding solutions fail in this environment due to two critical flaws: 1) severe sample inefficiency, where failed explorations under specific constraints yield no transferable knowledge for new budget-ROI combinations, and 2) limited generalization under constraint shifts, as they ignore physical relationships between constraints and bidding coefficients. To address this, we propose HALO: Hindsight-Augmented Learning for Online Auto-Bidding. HALO introduces a theoretically grounded hindsight mechanism that repurposes all explorations into training data for arbitrary constraint configuration via trajectory reorientation. Further, it employs B-spline functional representation, enabling continuous, derivative-aware bid mapping across constraint spaces. HALO ensures robust adaptation even when budget/ROI requirements differ drastically from training scenarios. Industrial dataset evaluations demonstrate the superiority of HALO in handling multi-scale constraints, reducing constraint violations while improving GMV.
\end{abstract}


\section{Introduction}
In today's digital marketing landscape, advertisers strategically deploy multi-channel campaigns to reach consumers and drive conversions cost-effectively. These efforts primarily follow two distinct approaches based on campaign objectives \cite{vakratsas1999advertising}: Brand Advertising focuses on long-term brand building through immersive formats like promotional videos, prioritizing awareness and brand loyalty; while Performance Advertising emphasizes immediate results via e-commerce product ads and social media promotions, measured through quantifiable KPIs including click-through rate (CTR), conversion rate (CVR), gross merchandise volume (GMV), and return on investment (ROI). Performance advertising dominates the digital advertising landscape, accounting for 60\% of industry expenditure compared to brand advertising's 40\% \cite{gatnercmo2023}. This spending disparity establishes performance advertising as the primary focus of our research.

The quantifiable nature of performance advertising objectives necessitates infrastructure capable of real-time decision-making at scale. To address the diverse needs of advertisers, platforms like TikTok and Google Ads \cite{tiktokad,googlead} utilize Real-Time Bidding (RTB) \cite{rtb}, the industry's standard infrastructure, which enables millisecond-level auctions for sequentially arriving ad impressions. Under this infrastructure, advertisers increasingly adopt multi-objective bidding strategies that maximize core outcomes (e.g., Gross Merchandise Volume (GMV)) while simultaneously satisfying multiple constraints, such as daily budgets and performance KPIs (e.g., ROI or Cost Per Action (CPA)). The bidding agent's main task is to dynamically compute a distinct bid for every opportunity, aligning real-time decisions with the advertiser’s constraints and objectives. This real-time constraint satisfaction requirement formally characterizes the auto-bidding problem into two strategy classes: Budget-Constrained Bidding (BCB) \cite{bcb} maximizes outcomes under only a total budget constraint. Multi-Constrained Bidding (MCB) \cite{mcb} maximizes outcomes under both budget constraint and additional KPI requirements (e.g., ROI $\ge$ X). This paper focuses on solving both BCB and MCB challenges within a unified framework.

The design of a unified framework faces fundamental challenges due to the heterogeneity of advertisers. Digital advertising platforms serve heterogeneous advertisers, from multinational enterprises to small businesses, resulting in vastly divergent budget allocations and ROI targets. This operational diversity creates complex constraint landscapes: enterprise advertisers may deploy budgets orders of magnitude larger than small businesses while pursuing fundamentally different performance thresholds. Critically, advertisers frequently adjust constraints mid-flight, implementing sudden budget surges or tightening ROI thresholds during peak campaigns. Empirical evidence reveals that about a third of campaigns make such adjustments within a single auction day, triggering severe performance volatility.

Existing work \cite{uscb,aigb,wu2018budget} predominantly operates under an implicit homogeneity assumption that constraint magnitudes exhibit limited variation across advertisers. This foundational limitation manifests in two critical deficiencies: 1. Severe exploration inefficiency: existing methods require exhaustive trajectory exploration to find a viable bidding trajectory for each specific constraint configuration, while generating numerous failed attempts. Crucially, these failed attempts yield no transferable knowledge when constraints shift, forcing complete re-exploration for each new configuration and incurring substantial exploration waste.
2. Fundamental generalization failures: due to neglecting contextual relationships between constraints and bidding coefficients, existing methods fail to capture physical correlations among constraints and bidding coefficients. This oversight creates critical modeling gaps that cause catastrophic extrapolation errors under both distribution shifts and constraint adjustments.
The compound effect of these deficiencies leads to concrete operational failures: significant extrapolation errors upon constraint shifts, resulting in unstable delivery pacing and suboptimal KPIs.

In this paper, we introduce an integrated framework that fundamentally transforms sample inefficiency and dynamic adaptation failures in constrained bidding. Drawing inspiration from human experiential learning, where even unsuccessful actions yield insights about alternative objectives, we establish a theoretically-grounded hindsight \cite{hindsight} sampling mechanism. This approach converts every exploration trajectory into productive training data, achieving near-perfect sample utilization efficiency. Crucially, we formally prove the efficacy of this mechanism and validate its optimization properties under constrained bidding conditions. Simultaneously, to resolve the extrapolation fragility, we abandon discrete point-wise mappings in favor of continuous functional mapping. Constraints-to-bid relationships are parameterized via adaptive B-spline \cite{Bspline}. This ensures robustness against constraint shift by preserving contextual correlations between multi-scale constraints and bidding. Empirical validations on the industrial dataset conclusively demonstrate the superiority of our framework. In general, our contributions can be summarized into three aspects:
\begin{itemize}
    \item We propose a theoretically-grounded hindsight sampling mechanism for the constrained bidding problem, which significantly improves sample utilization efficiency.
    \item We parameterize bid mappings via adaptive B-splines, achieving continuous adaptation to diverse constraint combinations while fundamentally enhancing extrapolation robustness against constraint shifts.
    \item Industrial validation on the real-world advertising dataset demonstrates the superiority of our solution.
\end{itemize}

\section{Problem Formulation}

In a bidding period (normally one day), consider a sequentially arriving set $\mathcal{I}$ of $N$ impression opportunities. Advertisers participate auction by submitting a bid $b_i$ for each opportunity. The impression is won if $b_i$ exceeds the highest competing bid $c_i$. The cost of this impression is also $c_i$ determined by generalized second-price auction mechanism \cite{gsp}. The advertiser's objective during the bidding period is to maximize the total value of winning impressions $\max \sum_{i \in \mathcal{I}} v_i x_i$, where $v_i \in \mathbb{R}^+$ is the private impression value (the same opportunity has different values for different advertisers), $x_i$ is a binary variable indicating whether the campaign wins impression $i$. 
Beyond that, advertisers must satisfy specific operational constraints. When subject solely to a budget constraint, the problem is termed Budget-Constrained Bidding (BCB):
\begin{align*}
&\max \quad \sum_{i \in \mathcal{I}} v_i x_i\\
& \begin{array}{r@{\quad}l@{}l@{\quad}l}
s.t.& \sum_{i \in \mathcal{I}} c_ix_i \leq B\\
\end{array} 
\end{align*}
If additional KPI constraint (such as Return on Investment (ROI)) is imposed alongside the budget, the problem is termed Multi-Constrained Bidding (MCB):
\begin{align*}
&\max \quad \sum_{i \in \mathcal{I}} v_i x_i\\
& \begin{array}{r@{\quad}l@{}l@{\quad}l}
s.t.& \sum_{i \in \mathcal{I}} c_ix_i \leq B\\
&\dfrac{\sum_{i \in \mathcal{I}} v_i x_i}{\sum_{i \in \mathcal{I}} c_ix_i} \geq r_{target} ,
\end{array} 
\end{align*}
where $r_{target}$ is the minimum allowable ROI.  

Previous work \cite{uscb} reformulated the constrained bidding problem as a linear programming problem, thus deriving the optimal bidding strategy for BCB and MCB scenarios.
\begin{align*}
&\text{BCB}:b_i = \beta_0 v_i  = \beta^{bcb} v_i\\
&\text{MCB}:b_i = (\beta_0 - \beta_1 r_{target})v_i  = \beta^{mcb}v_i
\end{align*}
where coefficient $\beta_0$ is relative to budget, and coefficient $\beta_1$ is relative to ROI. The fundamentally divergent optimal bidding formulations for BCB and MCB have prevented prior approaches from developing a unified optimization method for both paradigms. Our method decouples budget and ROI constraints to enable integrated processing of BCB and MCB within a single optimization architecture. 

Additionally, due to the unpredictability of future impressions during real-time bidding, a fixed global optimal bid coefficient does not exist; bidding agents need to dynamically adjust bidding coefficients throughout the auction period to ensure cumulative constraints are satisfied. However, a single bidding period typically involves billions of impression opportunities, and it is not feasible to adjust for every impression. To address this, the common solution \cite{wu2018budget} is to divide the bidding period into $T$ discrete decision steps. At each step, the agent derives a bidding coefficient $\beta$, which is then applied to all impression opportunities within that interval.

\section{Method}
We reformulate bidding strategies through dual-level innovations: at the data level, we propose a hindsight mechanism that enhances data efficiency; at the model level, we propose a parameterized B-spline to enable robust adaptation to multi-scale constraint configurations while reducing extrapolation errors. We now elaborate on these two innovations in detail. Due to space constraints, detailed proofs and formal definitions are provided in the appendix.

\subsection{Hindsight Mechanism}
Existing works \cite{uscb,wang2022roi} typically solve the constrained bidding problem by independently optimizing the bidding coefficient for each ($B,\ r_{target}$) pair, employing trial-and-error \cite{kaelbling1996reinforcement} to determine the optimal coefficient. However, these works suffer from a critical limitation of sample inefficiency. Specifically, the exploration data collected to optimize one budget, say $B_k$, contribute minimally to the optimization process for a different budget $B_m$. This inefficiency arises because the target constraint $B_m$ differs from the budget scale that governed the data collection history. 

To improve sample efficiency, we propose a hindsight mechanism. To demonstrate its effectiveness, we provide a series of theoretical guarantees. First, we define a naive strategy called Fixed Coefficient Strategy:

\begin{definition}[Fixed Coefficient Strategy (FCS)]
\label{define_fixed}At any decision step $\tau$, a fixed coefficient strategy commits to the coefficient $\beta$ is constant for all remaining steps $[\tau,T]$.
\end{definition}

FCS consolidates sequential decision-making across steps $[\tau,T]$ into a single optimization step, significantly reducing the decision space from $\mathcal{O}(T-\tau)$ to $\mathcal{O({\rm 1})}$. Furthermore, we define:

\begin{definition}[Omniscient Optimal Value]
The omniscient optimal value, denoted $V^{oracle}$, is the maximum achievable total value under perfect information with budget and ROI constraints.
\end{definition}

We prove that the maximum obtainable value can be achieved by FCS, exhibiting an ignorable error bound relative to the omniscient optimal value (solved by Mixed Integer Linear Programming \cite{floudas2005mixed}), as established in Lemma \ref{lemma_error_bound}.

\begin{lemma}[FCS's Error Bound]
\label{lemma_error_bound}
The total potential obtainable value $V^{fixed}$ using a fixed coefficient strategy satisfying
$$V^{fixed} > V^{oracle} - v_{max},$$
where $v_{max}$ is the maximum individual impression value.
\end{lemma}

Empirical evidence demonstrates that the value gap between the omniscient optimal value $V^{oracle}$ and the maximum individual impression value $v_{max}$ spans at least 2 orders of magnitude in practical advertising systems. Beyond that, under typical campaign conditions, the actual error remains much smaller than $v_{max}$, rendering the theoretical upper bound practically negligible. Consequently, we establish that FCS can achieve near-optimal performance, effectively approximating the omniscient optimum value while reducing optimization difficulty. 

\begin{lemma}[Optimality Condition for FCS Under BCB]
\label{lemma_optimal_condition}
 Under budget-constrained bidding with fixed coefficient Strategy, at any decision step $\tau$ with remaining budget $B_{\tau}$, if $\exists \beta^*$ such that
$$
\sum_{t=\tau}^T C_t(\beta^*) = B_{\tau},
$$
then $\beta^*$ is the optimal coefficient on $[\tau,T]$.
\end{lemma}

\newcommand{\pluseq}{\mathrel{+}=}
\begin{algorithm}
\caption{Hindsight Experience Collection}
\label{algo_collection}
\begin{algorithmic}[1]
\State \textbf{Input:} total timesteps $T$, advertiser feature, number of sample coefficients $L$, uniform distribution $\mathcal{U}$
\State \textbf{Output:} Dataset $\mathcal{D} = \{ (\mathbf{s}_i, \hat\beta_i, \text{cost}_i, \text{value}_i) \}$
\For{$l = 1$ to $L$}
\For{$\tau = 1$ to $T$}
    \State $\hat\beta_\tau \sim \mathcal{U}$ 

    \State $\mathbf{s}_\tau \leftarrow \text{compute\_features}$ 
    
    \State {$\hat{C}_\tau^T = 0$} \quad
    \State {$\hat{V}_\tau^T = 0$}
    
    \For{$t = \tau$ to $T$}
        \State $\text{Bid with } \hat\beta_\tau \text{ at time-step } t \text{ and get } {C_{t}} \text{, } {V_{t}};$
        \State {$\hat{C}_\tau^T \gets \hat{C}_\tau^T + C_{t}$}
        \State {$\hat{V}_\tau^T \gets \hat{V}_\tau^T + V_{t}$}
    \EndFor
    
    \State $\mathcal{D}.\text{append}\left( \mathbf{s}_\tau, \hat\beta_\tau, {\hat{C}_\tau^T }, {\hat{V}_\tau^T} \right)$
\EndFor
\EndFor
\State \Return $\mathcal{D}$
\end{algorithmic}
\end{algorithm}

We establish that under BCB, the FCS achieves optimality if and only if the budget is fully exhausted over the interval  $[\tau, T]$. Lemma \ref{lemma_optimal_condition} reveals a fundamental insight: multi-scale constrained bidding can be reformulated as a multi-objective trajectory alignment problem. Specifically, for any exploratory trajectory generated under a fixed coefficient $\beta$ (even when violating the target budget $B$), replacing $B$ with the realized cost $C_{realized}(\beta)$ transforms the trajectory into an optimal exploration for the posterior constraint configuration $B' = C_{realized}(\beta)$. Through this mechanism, every exploratory trial becomes an effective sample for some budget-constrained scenario. We formalize this paradigm as the Hindsight Mechanism, which establishes a bijection between historical explorations and valid constraint configurations. Through this mechanism, we can collect the hindsight experience through Algorithm \ref{algo_collection}.

\begin{definition}[Hindsight Experience]
\label{define_hindsight}
For any bidding trial over $[\tau,T]$ with fixed coefficient, we record
$$
\mathcal{H_{\tau}} = (\text{s}_\tau,\hat\beta_\tau,\hat{C}_{\tau}^T(\beta_\tau), \hat{V}_{\tau}^T(\beta_\tau)),
$$
as one hindsight experience tuple, where $\text{s}_\tau$ is the state feature ,$\hat{C}_{\tau}^T(\hat\beta_\tau)$ is the realized cost over $[\tau, T]$ and $\hat{V}_{\tau}^T(\hat\beta_\tau))$ is the realized value.
\end{definition}

After elucidating how budget constraint learning benefits from the Hindsight Mechanism, we now systematically extend this framework to ROI constraint satisfaction. When introducing ROI target, the bid coefficient $\beta^*$ given by Lemma \ref{lemma_optimal_condition} is likely to cause the realized ROI to violate ROI constraint. Define the expected ROI:
$$
R(\beta) = \mathbb{E}[\dfrac{\sum_{i \in \mathcal{I}} v_i x_i(\beta)}{\sum_{i \in \mathcal{I}} c_ix_i(\beta)}]
$$

\begin{lemma}[ROI Monotonicity]
\label{lem:roi_monotonicity}
The expected ROI function $R(\beta)$ is non-increasing in $\beta$.
\end{lemma}

Lemma \ref{lem:roi_monotonicity} indicates that the expected ROI typically increases as the bidding coefficient $\beta$ decreases. Therefore, when the $\beta^*$ derived from Lemma \ref{lemma_optimal_condition} results in the realized ROI violating the ROI constraint, we need to search for an adjustment factor $\alpha < 1$ to shade $\beta^*$. This scaling reduces bid intensity, filtering out low-value traffic while ensuring the adjusted ROI falls within the target threshold. These two core mechanisms transform constrained optimization into a dual-control system: Lemma \ref{lemma_optimal_condition} governs budget exhaustion, while $\alpha$ enforces ROI compliance through strategic bid suppression. Then we get the following lemma:

\begin{lemma}[MCB Coefficient Adjustment]
\label{lemma:roi_adjustment}
Under budget and ROI constraint with fixed coefficient Strategy, at any decision step $\tau$ with history cost $C_1^{\tau-1}$ and history value $V_1^{\tau-1}$, if $\exists \alpha\le1$ such that $\beta^{opt} = \alpha\beta^*$ applied on $[\tau,T]$ can make the total realized ROI equal to ROI constraint, then $\alpha$ satisfys
$$\beta^{*} \int_{\alpha}^{1} \left[ \gamma\left( \beta^{*} u \right) - r_{\mathrm{target}}  \eta\left( \beta^{*} u \right) \right]  du = \Delta_{\mathrm{ROI}}$$ 
where $\gamma = \frac{\partial V}{\partial \beta}$ is derivative of value with respect to coefficient and $\eta = \frac{\partial C}{\partial \beta}$ is derivative of cost with respect to coefficient, $\Delta_{\text{ROI}} = V_1^{\tau-1} + V_\tau^T(\beta^*) - r_{target}(C_1^{\tau-1} + C_\tau^T(\beta^*))$ with $\beta^*$.
\end{lemma}

Lemma \ref{lemma:roi_adjustment} establishes the functional form of the adjustment factor  $\alpha$ as dependent on seven key determinants: historical aggregate cost $C_1^{\tau-1}$, historical cumulative value $V_1^{\tau-1}$, future total cost with $\beta^*$ ($C_\tau^T(\beta^*)$), future total value with $\beta^*$ ($V_\tau^T(\beta^*)$), the derivative of  value ($\gamma$), the derivative of cost ($\eta$), and the target ROI ($r_{target}$). This relationship is formally expressed as:
$$
\alpha = f(C_1^{\tau-1},V_1^{\tau-1}, C_\tau^T(\beta^*),V_\tau^T(\beta^*),\gamma,\eta,r_{target}).
$$

The formulation exposes a fundamental limitation in traditional bidding methods: their failure to explicitly model the value derivative $\gamma$ and cost derivative $\eta$ critically undermines bidding performance. This omission prevents accurate sensitivity quantification of cost-value responses to coefficient variations, directly causing ROI constraint violations. Consequently, we fundamentally depart from conventional point-wise mapping approaches that naively fit constraint-to-coefficient mappings without physical interpretability. Instead, we establish continuous functional mappings that intrinsically encode derivative relationships. 

Specifically, we model the bid coefficient function $\beta_{\tau} = f_{\theta}(B_\tau)$ and future value function $V_{\tau}^T = f_\phi(B_\tau)$, where $f_\theta,f_\phi : \mathbb{R^+} \rightarrow \mathbb{R^+}$ are parameterized continuous functions learned from hindsight experiences $\mathcal{D} = \{\mathcal{H_{\tau}}\}$. It is worth noting that we transition from coefficient-centric (in Lemma \ref{lemma:roi_adjustment}) to budget-centric parameterization, because budget serves as the directly observable state variable that intrinsically defines the operational constraint boundary.
This transformation preserves derivative relationships through the chain rule \cite{swokowski1979calculus} $\frac{\partial V}{\partial \beta} = \frac{\partial f_{\phi}}{\partial B} \cdot \frac{\partial B}{\partial \beta} = \frac{\partial f_{\phi}}{\partial B} \cdot \left( \frac{\partial f_{\theta}}{\partial B} \right)^{-1}$. This functional mapping intrinsically encodes gradient dynamics while ensuring operational feasibility, enabling derivative-informed computation of the adjustment factor $\alpha$ for precise constraint satisfaction. This fundamental idea delivers a breakthrough advance and makes our method outperform state-of-the-art alternatives

Combining the above lemma and definition, we can get the theorem below.
\begin{theorem}[Optimal Bidding Strategy]
\label{thm:elasticity_cost_adjustment}
At any step $\tau$, with remaining budget $B_\tau$, target ROI $r_{target}$,historical cost $C_1^{\tau-1}$,historical value $V_1^{\tau-1}$, solve for $\beta_\tau = f_\theta(B_\tau)$, $V_{\tau}^T= f_\phi(B_\tau)$ and calculate $\Delta_{\text{ROI}} = V_1^{\tau-1} + V_{\tau}^T - r_{target}(C_1^{\tau-1} + B_\tau)$
\begin{itemize}
    \item if $\Delta_{\text{ROI}} \ge 0$,  $\beta_\tau$ is the optimal bidding coefficient.
    \item if $\Delta_{\text{ROI}} < 0$, $\alpha \beta_\tau$ is the optimal bidding coefficient.
\end{itemize}
\end{theorem}

\subsection{B-spine Policy}

As we discuss above, the continuous functions of value and cost are important to compute the adjustment factor $\alpha$ in Lemma \ref{lemma:roi_adjustment}. However, the properties of continuous value/cost functions exhibit significant variation across bidding environments. To ensure modeling tractability while preserving fidelity, functional assumptions become necessary because the assumptions critically govern both approximation accuracy and the precision of $\alpha$. To determine optimal functional representations, we categorize derivatives into three types with corresponding functional assumptions:

\begin{figure}
    \centering
    \subfloat[$\beta^*$ vs $B_\tau$]
    {
    \includegraphics[width=0.4\linewidth]{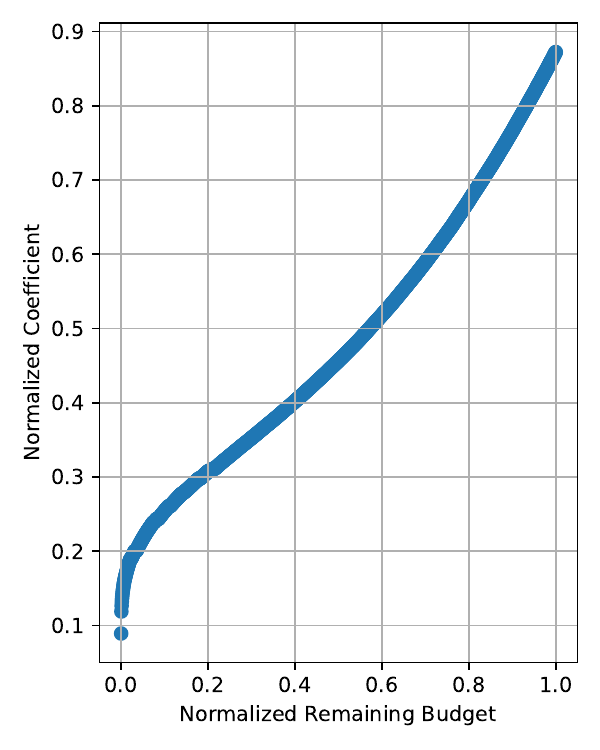}
    }
    \subfloat[Four paradigms]
    {
    \includegraphics[width=0.41\linewidth]{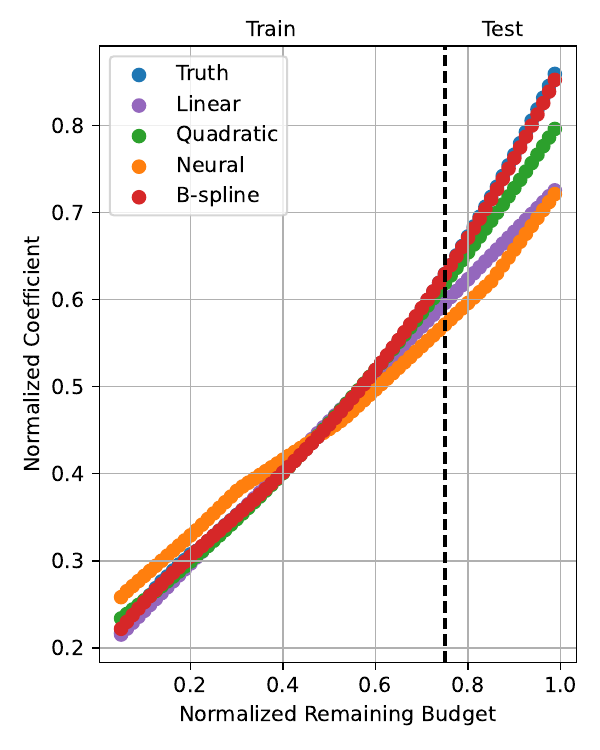}
    }
    \caption{(a) For a specific advertiser at decision step $\tau$, the relationship between remaining budget $B_{\tau}$ and $\beta^*$. (b) Partition the budget domain into training and testing to evaluate the generalization capabilities of four modeling paradigms under extrapolation conditions.}
    \label{fig:flow shape}
    \vspace{-1.5em}
\end{figure}

\vspace{-0.2em}
\begin{enumerate}
\setlength\itemsep{-0.3em}
\item Constant Derivative ($\frac{\partial \beta}{\partial B} =\text{constant}$): model the relationship function via linear polynomial $f_{\theta}(B_{\tau}) = \theta_0 + \theta_1 B_{\tau}$.

\item Linear Derivative ($\frac{\partial \beta}{\partial B} = a + b\cdot B$): model the relationship function via quadratic polynomial $f_{\theta}(B_{\tau}) = \theta_0 + \theta_1 B_{\tau} + \theta_2 B_{\tau}^2$.

\item Nonlinear Derivative ($\frac{\partial \beta}{\partial B} = g(B)$): model the relationship function via neural network, $f_{\theta}(B_{\tau}) = \text{MLP}(B_{\tau}; \theta)$
\end{enumerate}

This taxonomy directly dictates computational methodology: constant derivatives permit closed-form $\alpha$ through algebraic manipulation, while linear and nonlinear derivatives need to compute $\alpha$ via iterative numerical techniques \cite{ypma1995historical}. The choice of assumptions inherently balances expressivity against computational complexity. Linear and quadratic polynomial functions impose a strong parametric prior that fails to capture the inherent heterogeneity in real-world online bidding environments. As empirically demonstrated in Figure \ref{fig:flow shape} (a), the relationship between remaining budget and optimal coefficient exhibits non-stationary complexity that cannot be adequately modeled by rigid functional forms. While neural networks offer greater flexibility, their decision boundaries exhibit fragility and high sensitivity to input noise \cite{ghorbani2019interpretation}.

We therefore propose modeling the mapping relationship using B-splines \cite{liu2024kan}, which provide adaptive local control through piecewise polynomial segments connected at knots. Crucially, this formulation ensures local modification independence, unlike neural networks. This locality property enables precise approximation of complex mappings while maintaining structural stability. The comparative analysis of four modeling paradigms (linear polynomial, quadratic polynomial, neural network, and B-spline) is visually summarized in Figure \ref{fig:flow shape} (b), providing empirical validation of B-spline's superiority.

\begin{figure}
    \centering
    \includegraphics[width=1.0\linewidth]{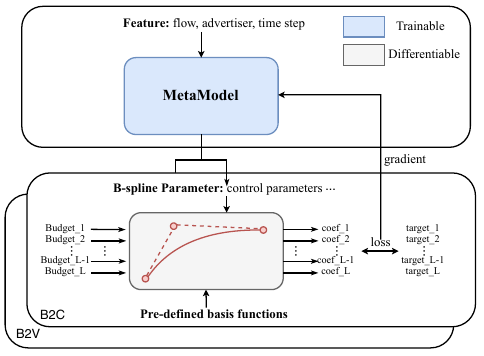}
    \caption{The training process of the Hindsight-Augmented B-Spline Policy, the blue Meta-Model is trainable, while the gray B-splines curve determined by the output of the Meta-Model. This B-spline module supports differentiable gradient backpropagation but remains non-trainable.}
    \label{fig:meta_model}
    \vspace{-1.0em}
    
\end{figure}

Based upon the analysis, we propose a Dual B-Spline Policy for the constrained bidding problem, where a shared MetaModel processes multi-dimensional features (traffic patterns, advertiser features, left time) to generate parameters ($\theta,\phi = \text{MetaModel}(\text{s}_\tau)$) for two B-spline curves: $\beta_{\tau} = \text{B-Spline}_\theta(B_\tau)$, $V_{\tau} = \text{B-Spline}_\phi(B_\tau)$ where $\theta$ and $\phi$ are learnable control points for their respective splines. We employ a meta-learning-inspired \cite{hospedales2021meta} paradigm to train the MetaModel, wherein predicted bidding coefficients $\beta_\tau$ and future values $V_{\tau}^T$ are simultaneously evaluated against ground-truth optimal labels (obtained by Algo. \ref{algo_collection}) to compute a composite loss function:
$$\mathcal{L} = \frac{\text{1}}{|\mathcal{S}|}\sum_{s_\tau \in \mathcal{S}} \frac{1}{L} \sum_{k=1}^L(\|\beta_{\tau, k} - \hat\beta_{\tau, k} \| ^{\text{2}} + \|V_{\tau, k} - \hat{V}_{\tau, k} \| ^{\text{2}}).$$
Gradients are backpropagated through the dual spline modules to update the MetaModel end-to-end, enabling joint functional optimization of both budget-to-coefficient (B2C) and budget-to-value (B2V) mappings. As illustrated in Figure \ref{fig:meta_model}, the training process performs parallelized evaluation of mapping at $L$ anchor budgets $\{B_k\}_{k=1}^L$ under identical contextual features. This bath comparison against the corresponding coefficient target and value targets holistically quantifies the B-spline's functional approximation quality across the budget domain. This design breaks through the limitations of traditional point-wise optimization, accelerating convergence through Multi-Anchor Joint Gradients while enhancing model generalization \cite{finn2017model}.

\begin{algorithm}[t]
\caption{Optimal bidding Algorithm}
\label{alg:cost_gradient_alpha}
\begin{algorithmic}
\Statex \begin{center}
\includegraphics[width=0.8\linewidth]{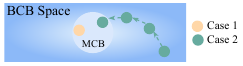}\end{center}
\State \textbf{Input:} current step $\tau$, state $\text{s}_\tau$, remaining budget $B_\tau$, target ROI $r_{\text{target}}$, historical cost$C_1^{\tau-1}$, historical value $V_1^{\tau-1}$, MetaModel $f_\Theta$, update scale $\lambda$, max iterations $K$
\State \textbf{Output:} optimal bidding coefficient $\beta^{\text{opt}}_\tau$

\State $\theta, \phi \leftarrow f_\Theta(\text{s}_\tau)$
\State $\beta_\tau \leftarrow f_\theta(B_\tau)$
\State $V_\tau^T \leftarrow f_\phi(B_\tau)$

\State $C_{\text{target}} \leftarrow B_\tau$ 
\State $\Delta_{\text{ROI}} \leftarrow V_1^{\tau-1} + V_\tau^T - r_{\text{target}}(C_1^{\tau-1} + C_{\text{target}})$

\If{$\Delta_{\text{ROI}} \geq 0$}
    \State $\beta^{\text{opt}}_\tau \leftarrow \beta_\tau$ \Comment{Case 1}
\Else
    \For{$k = 1$ to $K$}
        \State $g \leftarrow \frac{\partial (\Delta_{\text{ROI}})}{\partial C_{\text{target}}} $
        
        \State $C_{target} \leftarrow C_{target} + \lambda \cdot g $
        \State $V_\tau^T \leftarrow f_\phi(C_{target})$
        \State $\Delta_{\text{ROI}} \leftarrow V_1^{\tau-1} + V_\tau^T - r_{\text{target}}(C_1^{\tau-1} + C_{\text{target}})$
        \If{$\Delta_{\text{ROI}} \ge 0$}
            \State \textbf{break}
        \EndIf
    \EndFor
    \State $\beta^{\text{opt}}_\tau \leftarrow  f_\theta(C_{target})$ \Comment{Case 2}
\EndIf
\State \Return $\beta^{\text{opt}}_\tau$
\end{algorithmic}
\end{algorithm}

In practical deployment (Algo. \ref{alg:cost_gradient_alpha}), at each decision timestep $\tau$, we input multi-dimensional features into the MetaModel. This MetaModel generates control parameters for B-spline curves. Cost-to-Coefficient B-spline subsequently maps the advertiser’s remaining budget $C_{target} = B_\tau$ to a bidding coefficient $\beta_\tau$ under BCB assumptions. Then we leverage Cost-to-Value B-spline to map the target cost $C_{target}$ to the expected future value $V_\tau^T$. If $\Delta_{\text{ROI}} = V_1^{\tau-1} + V_\tau^T - r_{\text{target}}(C_1^{\tau-1} + C_{\text{target}}) \ge 0$, $\beta_\tau$ inherently fall in the feasible solution space of MCB and is deployed without modification. Conversely, if 
$\Delta_{\text{ROI}} < 0 $, we initiate a gradient-based correction protocol \cite{zhang2023evaluating} to dynamically adjust the target cost $C_{target}$ until ROI feasibility is achieved:
$$
\psi(C_{target}^{k}) = C_{target}^{k} + \lambda \cdot \frac{\partial}{\partial C_{target}^{k}}\left[\Delta_{\text{ROI}}\right]^{-}
$$
where $\lambda$ is the update scale. Then we get $\beta^{opt}_\tau = f_\theta(C_{target})$. 

This framework intrinsically supports both BCB and MCB bidding modes. Furthermore, the robust extrapolation capability of B-spline parameterization enables real-time adaptation to multi-scale constraint configurations and constraint modifications mid-period. Such adaptability guarantees continuous constraint satisfaction under non-stationary advertising environments while maintaining solution optimality.

\begin{table*}[h]
\centering
\small
\caption{\textbf{Evaluation results.} For value-related and budget-related metrics, the top-performing method is highlighted in boldface, while ROI constraint compliance rates exceeding 90\% are denoted with underlining.}
\label{table1}
\renewcommand{\arraystretch}{1.15}
\begin{tabular}{c|cc|ccccccccc}
\hline \hline
\textbf{Mode}                                       & \textbf{Budget}                 & \textbf{Metrics}         & \textbf{BC}    & \textbf{DT}    & \textbf{IQL}   & \textbf{PID}   & \textbf{MPC}   & \textbf{LP}    & \textbf{GAVE}  & \textbf{USCB}  & \textbf{Our}   \\ \hline
\multicolumn{1}{c|}{\multirow{15}{*}{MCB}} & \multirow{2}{*}{50\%}  & Conv     & 15.31&  16.54&  13.44&  8.1&  10.71&  11.62&  15.08&  14.92&  \textbf{18.58}    \\ \multicolumn{1}{c|}{}                      &                        & C.R. & \underline{0.938} & 0.854 & 0.792 & 0.729 & 0.708 & 0.833 & 0.833 & \underline{0.938} & \underline{0.917} \\ \cmidrule(r){2-3} \cmidrule(r){4-12}
\multicolumn{1}{c|}{}                      & \multirow{2}{*}{75\%}  & Conv     & 19.85&  22.96&  18.25&  10.33&  15.29&  17.19&  19.21&  19.69&  \textbf{25.46}  \\
\multicolumn{1}{c|}{}                      &                        & C.R. & \underline{0.938} & 0.854 & 0.771 & 0.667 & 0.688 & 0.812 & 0.812 & \underline{0.979} & \underline{0.938} \\
\cmidrule(r){2-3} \cmidrule(r){4-12} 
\multicolumn{1}{c|}{}                      & \multirow{2}{*}{100\%} & Conv     & 23.15&  28.5&  22.31&  13.1&  18.67&  22.1&  23.58&  23.5&  \textbf{31.35} \\
\multicolumn{1}{c|}{}                      &                        & C.R. & \underline{0.938} & 0.854 & 0.75  & 0.729 & 0.729 & 0.792 & 0.833 & \underline{0.938} & \underline{0.938} \\
\cmidrule(r){2-3} \cmidrule(r){4-12} 
\multicolumn{1}{c|}{}                      & \multirow{2}{*}{125\%} & Conv     & 25.94&  33.81&  25.77&  17.46&  24.23&  25.56&  28.48&  27.6&  \textbf{36.06}  \\
\multicolumn{1}{c|}{}                      &                        & C.R. & \underline{0.938} & 0.875 & 0.729 & 0.646 & 0.708 & 0.833 & 0.875 & \underline{0.979} & \underline{0.917} \\
\cmidrule(r){2-3} \cmidrule(r){4-12} 
\multicolumn{1}{c|}{}                      & \multirow{2}{*}{150\%} & Conv     & 27.6&  37.73&  29.98&  17.75&  29.6&  31.65&  29.58&  29.5&  \textbf{42.25}  \\
\multicolumn{1}{c|}{}                      &                        & C.R. & \underline{0.938} & 0.875 & 0.729 & 0.792 & 0.875 & \underline{0.917} & 0.833 & \underline{0.938} & \underline{0.938} \\
\hline \hline
\multicolumn{1}{c|}{\multirow{10}{*}{BCB}} & \multirow{2}{*}{50\%}  & Conv     & 16.31  & 19.94  & 18.71  & 11.33  & 16.54  & 15.35  & 16.14   & 14.25  & \textbf{24.15}  \\ 
\multicolumn{1}{c|}{}                      &                        & C/B     & 0.594 & 0.842 & 0.827 & 0.739 & 0.899 & 0.591 & 0.611 & 0.624 & \textbf{0.992} \\ \cmidrule(r){2-3} \cmidrule(r){4-12}
\multicolumn{1}{c|}{}                      & \multirow{2}{*}{75\%}  & Conv     & 21.6&  27.62&  26.81&  16.44&  21.04&  23.31&  19.27&  19.23&  \textbf{35.21}  \\
\multicolumn{1}{c|}{}                      &                        & C/B     & 0.508 & 0.759 & 0.792 & 0.715 & 0.847 & 0.612 & 0.575 & 0.547 & \textbf{0.993} \\ \cmidrule(r){2-3} \cmidrule(r){4-12}
\multicolumn{1}{c|}{}                      & \multirow{2}{*}{100\%} & Conv     & 23.42&  30.67&  34.21&  16.9&  27.06&  27.56&  24.35&  23.31&  \textbf{43.83} \\
\multicolumn{1}{c|}{}                      &                        & C/B     & 0.448 & 0.692 & 0.762 & 0.572 & 0.820 & 0.621 & 0.543 & 0.510 & \textbf{0.992} \\ \cmidrule(r){2-3} \cmidrule(r){4-12}
\multicolumn{1}{c|}{}                      & \multirow{2}{*}{125\%} & Conv     & 26.88&  36.81&  38.69&  23.75&  32.06&  33.23&  29.17&  28.9&  \textbf{51.27}  \\
\multicolumn{1}{c|}{}                      &                        & C/B     & 0.402 & 0.640 & 0.737 & 0.632 & 0.772 & 0.606 & 0.485 & 0.470 & \textbf{0.990} \\ \cmidrule(r){2-3} \cmidrule(r){4-12}
\multicolumn{1}{c|}{}                      & \multirow{2}{*}{150\%} & Conv     & 30.17&  39.04&  42.94&  26.19&  37.02&  35.35&  30.81&  30.38&  \textbf{58.96}  \\
\multicolumn{1}{c|}{}                      &                        & C/B     & 0.348 & 0.591 & 0.715 & 0.610 & 0.761 & 0.559 & 0.426 & 0.416 & \textbf{0.989} \\ \hline \hline
\end{tabular}
\end{table*}

\section{Experiments}
\subsection{Experiments Setup}
\textbf{Dataset.} We evaluate our approach using AuctionNet \cite{su2024auctionnet}, a standardized large-scale advertising bidding dataset that abstracts real-world auction dynamics while leveraging authentic bidding data to provide a unified benchmark for evaluating diverse bidding strategies. This dataset includes 21 days of bidding records, with each day containing approximately 500,000 impression opportunities temporally partitioned into 48 distinct bidding time-steps. Critically, each daily auction period involves participation from 48 unique advertisers with different scales of constraints. All evaluation methods emulate 48 distinct advertisers, conducting bidding under heterogeneous constraint magnitudes. And we partition the dataset into a 14-day training set (678 bidding trajectories ) and a 7-day test set (339 trajectories).

\noindent\textbf{Baselines.} We benchmark our approach against three categories of existing methods:
\begin{enumerate}
    \setlength\itemsep{-0.3em}
    \item \textbf{Rule-based Approaches}: 
    \vspace{-0.5em}
    \begin{itemize}
        \setlength\itemsep{-0.2em}
        \item \textbf{PID} \cite{johnson2005pid}: Adjusts bids dynamically using real-time error feedback.
        \item \textbf{MPC} \cite{mpc}: Solves constrained optimization over a finite time horizon.
        \item \textbf{LP} \cite{lp}: Formulates bidding as a constrained linear programming problem.
    \end{itemize}
    \item \textbf{Reinforcement Learning (RL) Methods} \cite{rl}:
    \vspace{-0.5em}
    \begin{itemize}
    \setlength\itemsep{-0.2em}
        \item \textbf{IQL} \cite{iql}: Offline RL leveraging value function approximation.
        \item \textbf{USCB} \cite{uscb}: A method for learning to dynamically adjust bid coefficients using the DDPG \cite{ddpg} algorithm.
    \end{itemize}
    \item \textbf{Imitation \cite{hussein2017imitation} \& Generative \cite{bond2021deep} Approaches}:
    \vspace{-0.5em}
    \begin{itemize}
    \setlength\itemsep{-0.2em}
        \item \textbf{BC} \cite{bc}: Mimics offline bidding decisions through supervised learning.
        \item \textbf{DT} \cite{dt}: Generates actions using return-conditioned sequence modeling.
        \item \textbf{GAVE} \cite{gave}: DT-based model, accommodates various advertising objectives through a score-based Return-To-Go (RTG) module.
    \end{itemize}
\end{enumerate}

\noindent\textbf{Evaluation Metrics.} For each advertiser, the core objective within every auction period is to maximize the aggregate value (conversions here) of winning impressions while adhering to predefined constraints. We categorize the evaluation paradigm into BCB and MCB. Under MCB settings, we formalize two evaluation metrics:\textbf{ Conversions (Conv)}: Average conversions exclusively during periods where ROI constraints are satisfied. Conversions from constraint-violating periods are strictly excluded. \textbf{Compliance Rate (C.R.):} Proportion of evaluation periods where ROI constraints are fully satisfied. For BCB scenarios (absent ROI constraints), two metrics are defined: \textbf{Conversions (Conv):} Average conversions. \textbf{Cost/Budget Ratio (C/B):} Measures actual expenditure relative to allocated budget.

\noindent\textbf{Implementation Details.} All experiments are conducted on NVIDIA V100 GPU. The sample number of $\beta$ in Algo. \ref{algo_collection} is 10. We use Adam optimizer to optimize MetaModel parameters with a learning rate of $1e^{-2}$. The state feature specifications and additional model hyperparameters are provided in the appendix. To ensure statistical significance, we conducted five independent experiments and report aggregated performance metrics.

\begin{figure}
    \centering
    \includegraphics[width=0.8\linewidth]{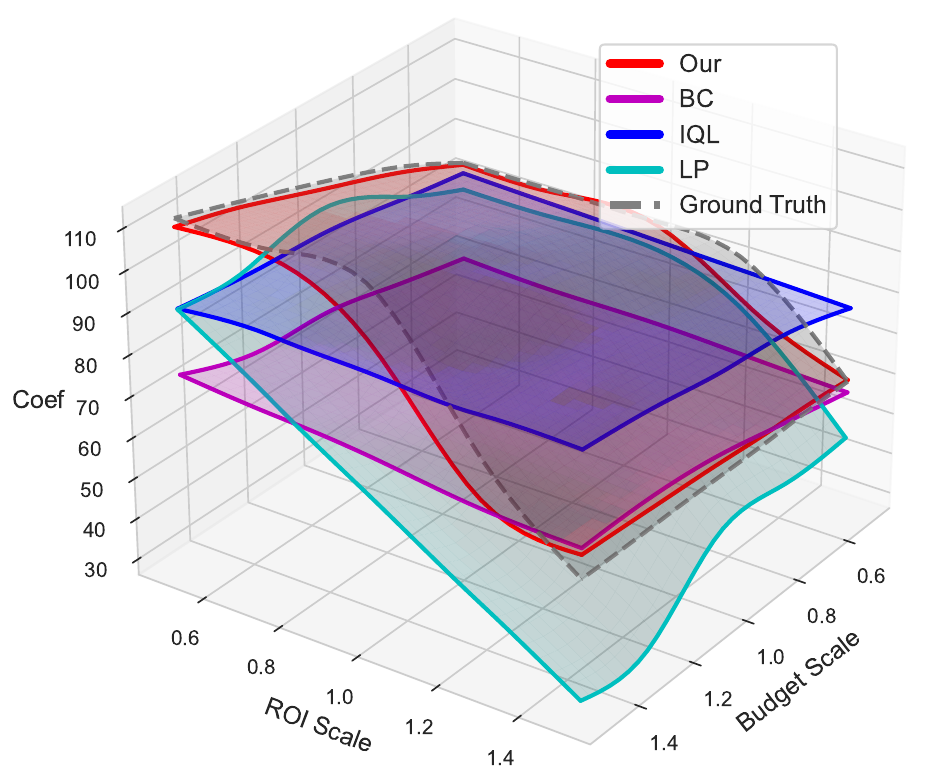}
    \caption{The variation of average bidding coefficients with different budget and ROI constraints. Each surface is interpolated from a grid of 25 discrete points (5 budget scales × 5 ROI scales).}
    \label{fig:3d}
    \vspace{-1.5em}
\end{figure}

\subsection{Evaluation Results} 
\textbf{Main Result.} Table \ref{table1} presents a comprehensive performance comparison of our method against baselines across varying bidding modes and budget scales. The results demonstrate that our approach achieves dominant performance across all bidding modes and budget scales. Under MCB settings, our method consistently maintains constraint compliance rates exceeding 90\% while maximizing value delivery. In BCB scenarios, our method exhibits precise budget-scale adaptability and achieves near-perfect budget utilization with cost/budget ratios around 99\% regardless of budget scales. The dual-mode efficacy establishes unprecedented robustness in real-world advertising operations. Traditional rule-based methods (PID, MPC, LP) underperformed due to relying exclusively on real-time error feedback. Reinforcement learning (IQL, USCB) and generative approaches (DT, GAVE\footnote{GAVE requires 400k+ trajectories in original work, due to training time constraint, our dataset size is insufficient to meet its training requirements.}) outperformed imitation methods (BC), as both categories incorporate guidance mechanisms during training or inference. However, imitation methods cannot intrinsically discern suboptimal actions, resulting in indiscriminate behavioral replication without value discrimination.

\noindent\textbf{Constraint Awareness.} Our method demonstrates superior constraint-awareness in dynamically adapting coefficients to varying budget and ROI configurations, as visualized in Fig. \ref{fig:3d}. While data-driven approaches (BC, IQL) exhibit limited adjustment responsiveness due to the limitation of the optimization algorithm, model-based methods (LP) fail to achieve competitive scores despite detecting constraint shift due to the lack of environmental dynamics modeling capabilities for fine adaptation.

\begin{figure}
    \centering
    \includegraphics[width=0.8\linewidth]{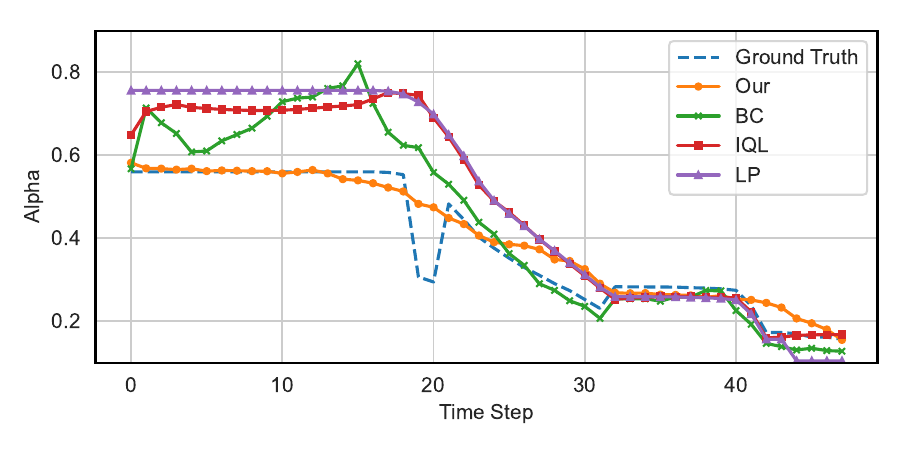}
    \caption{During a bidding period, the variation in $\alpha$ (defined in Lemma \ref{lemma:roi_adjustment}).}
    \label{fig:alpha compare}
    \vspace{-1em}
\end{figure}

\noindent\textbf{Coefficient Adjustment.} As shown in Fig. \ref{fig:alpha compare}, our method achieves precise $\alpha$ prediction ($\beta^{opt} / \beta^*$), enabling proactive bid suppression during early auction stages to ensure terminal constraint satisfaction. In contrast, baseline approaches exhibit delayed response, resulting in violating the ROI constraint. This occurs because late-period corrective adjustments cannot offset the irrecoverable ROI deficit accumulated during initial phases. This mechanistic divergence manifests quantitatively in $\text{ROI}_{realized} / \text{ROI}_{target}$ ratios: Our: 1.0204; BC: 0.9265; IQL: 0.7657; LP: 0.7632.

\vspace{-1em}
\section{Conclusions}
This paper introduces a theoretically grounded framework for multi-scale constrained advertising bidding. We first derive the optimality conditions under budget constraints, which inspires the development of the Hindsight Mechanism. Subsequently, we establish that the adjustment factor $\alpha$ for ROI constraints depends critically on the continuous relationship of cost and value. This insight motivates us to use the parameterized B-spline to model continuous value/cost relationships. The resulting framework accurately predicts both base bidding coefficients and adjustment factors to satisfy budget and ROI constraints simultaneously. Comprehensive experiments demonstrate significant improvements over state-of-the-art methods on advertising datasets.

\appendix
\bibliography{aaai2026}

\section{The Algorithm of obtaining $V^{oracle}$}

Consider the following ad examples:

\begin{table}[h]
\centering
\caption{Ad Selection Example ($B=4$, $r_{\text{target}}=1.0$)}
\begin{tabular}{c|c|c|c}
Opportunity & $v_i$ & $c_i$ & Efficiency ($v_i/c_i$) \\ \hline
A & 4 & 3 & 1.33 \\
B & 3 & 2 & 1.50 \\
C & 2 & 1 & 2.00 \\
\end{tabular}
\end{table}

The fixed coefficient greedy strategy  selects ads in descending efficiency order:
\begin{enumerate}
\item Select C: cost=1, value=2, remaining budget=3
\item Select B: cost=2, value=3, total value=5, remaining budget=1
\item Skip A (cost=3 $>$ remaining budget=1)
\textbf{Final value: 5}
\end{enumerate}

The globally optimal solution combines different ads: Select A+C: cost=3+1=4, value=4+2=6 \textbf{Optimal value: 6}

The greedy approach fails because:
1. \textbf{Budget fragmentation}: Selecting B first leaves insufficient budget for higher-value A
2. \textbf{Combinatorial effects}: A+C synergy (value=6) exceeds B+C (value=5)
3. \textbf{Efficiency deception}: A's lower efficiency (1.33) hides its value potential in combination

So the fixed coefficient strategy can not always obtain the global optimal value, we give the global optimal solution (Mixed Integer Linear Programming) for the constrained bidding problem in Algo. \ref{alg:milp_dual_constrained} However, this multi-variable algorithm is extremely inefficient.

\begin{algorithm}
\caption{MILP Solver for $V^{\text{oracle}}$}
\label{alg:milp_dual_constrained}
\begin{algorithmic}[1]
\Require Value vector $\mathbf{v}$, cost vector $\mathbf{c}$, budget $B$, target ROI $r_{\text{target}}$
\Ensure Optimal value $V^*$, binary selection vector $\mathbf{x}^* \in \{0,1\}^n$
\Statex
\Function{MILP}{$\mathbf{v}, \mathbf{c}, B, r_{\text{target}}$}
\State \textbf{Define binary variables:} $x_i \in \{0,1\} \quad \forall i=1,\dots,n$
\State \textbf{Objective:} $\max \sum_{i=1}^n v_i x_i$
\State \textbf{Subject to:}
\State $\quad \sum_{i=1}^n c_i x_i \leq B$ \Comment{Budget constraint}
\State $\quad \sum_{i=1}^n (v_i - r_{\text{target}} c_i) x_i \geq 0$ \Comment{ROI constraint}
\Statex
\State \textbf{Solve MILP} using branch-and-cut
\State \Return $V^* = \text{objval}, \mathbf{x}^* = [x_i]$
\EndFunction
\end{algorithmic}
\end{algorithm}

\section{Proof of Lemma 1}
\begin{lemma}
For any fixed coefficient bidding strategy with optimal coefficient $\beta^*$, the achieved total value $V^{\text{fixed}}$ satisfies:
\begin{equation}
V^{\text{fixed}} > V^{\text{oracle}} - v_{\max}
\end{equation}
where $V^{\text{oracle}}$ is the global optimal value satisfying both budget and ROI constraints, and $v_{\max} = \max_{i \in \mathcal{I}} v_i$ is the maximum value of any single impression opportunity.
\end{lemma}

\newcommand{\rtext}{r_{\text{target}}}
\newcommand{\Vfrac}{V^{\text{frac}}}
\newcommand{\Voracle}{V^{\text{oracle}}}
\newcommand{\Vfixed}{V^{\text{fixed}}}
\newcommand{\Vgreedy}{V^{\text{greedy}}}
\newcommand{\vmax}{v_{\text{max}}}
\begin{proof}
Define the fractional relaxation problem:
\begin{align*}
\max &\sum v_i y_i \\
\text{s.t.} &\sum c_i y_i \leq B \\
&\sum (v_i - \rtext c_i) y_i \geq 0 \\
&0 \leq y_i \leq 1
\end{align*}
has optimal value $\Vfrac$. Since any integer solution is feasible for this relaxation ($x_i \in \{0,1\} \subset [0,1]$), we have:
\begin{equation}
\Voracle \leq \Vfrac
\end{equation}
This follows from the fundamental principle of linear programming relaxation: expanding the feasible region cannot decrease the optimal value.

The fractional optimum is achieved by:
\begin{equation}
\Vfrac = \sum_{i=1}^k v_i + r \cdot v_{k+1}
\end{equation}
where:
\begin{itemize}
\item Ads are sorted by efficiency $e_i = v_i/c_i$ descending
\item $k$ is the largest integer satisfying:
  \begin{align}
  \sum_{i=1}^k c_i \leq B \quad \text{and} \quad \frac{\sum_{i=1}^k v_i}{\sum_{i=1}^k c_i} \geq \rtext
  \end{align}
\item $r \in (0,1)$ is the solution to:
  \begin{align}
  \sum_{i=1}^k c_i + r \cdot c_{k+1} &= B \\
  \frac{\sum_{i=1}^k v_i + r \cdot v_{k+1}}{\sum_{i=1}^k c_i + r \cdot c_{k+1}} &\geq \rtext
  \end{align}
\end{itemize}
This structure is optimal because:
\begin{enumerate}
\item Sorting by $e_i$ maximizes value per cost
\item Full selection of first $k$ ads uses the most efficient impressions
\item Fractional selection of ad $k+1$ exactly exhausts budget while maintaining ROI
\end{enumerate}

The fixed coefficient strategy:
\begin{enumerate}
\item Sort ads by $e_i$ descending
\item Select maximal set $S = \{1,\dots,m\}$ satisfying:
  \begin{align}
  \sum_{i \in S} c_i \leq B \quad \text{and} \quad \frac{\sum_{i \in S} v_i}{\sum_{i \in S} c_i} \geq \rtext
  \end{align}
\end{enumerate}
yields value $\Vfixed = \sum_{i=1}^m v_i$. By construction, $m \leq k$ because:
the fixed coefficient strategy selects a subset of the fractionally optimal ads: $S \subseteq \{1,\dots,k+1\}$

Suppose ad $j > k+1$ is selected. Since $e_j \leq e_{k+1} \leq \cdots \leq e_1$, adding $j$ would:
\begin{itemize}
\item Violate ROI if $e_j < \rtext$ (cumulative ROI decreases)
\item Or be suboptimal (more efficient ads available)
\end{itemize}
Thus the strategy only selects from $\{1,\dots,k+1\}$.

The value difference is:
\begin{align}
\Vfrac - \Vfixed &= \left( \sum_{i=1}^k v_i + r \cdot v_{k+1} \right) - \sum_{i=1}^m v_i
\end{align}
Case analysis:
\begin{itemize}
\item If $m = k$ (strategy takes all full ads):
  \[
  \Vfrac - \Vfixed = r \cdot v_{k+1}
  \]
\item If $m < k$ (ROI constraint prevents full selection):
  \begin{align}
  \Vfrac - \Vfixed &= \sum_{i=m+1}^k v_i + r \cdot v_{k+1} \\
  &< \sum_{i=m+1}^{k+1} v_i \quad \text{(since $r < 1$)}
  \end{align}
\end{itemize}
In both cases:
\begin{align}
\Vfrac - \Vfixed &< v_{k+1} + \sum_{i=m+1}^k v_i \\
&\leq (k - m + 1) \cdot \max_{i \geq m+1} v_i \\
&\leq n \cdot \vmax \quad \text{(worst-case)}
\end{align}
But tighter bound exists: $\Vfrac - \Vfixed < \vmax$

The strategy selects maximal $m$ satisfying constraints. If $m < k$, then:
\[
\frac{\sum_{i=1}^m v_i}{\sum_{i=1}^m c_i} \geq \rtext > \frac{\sum_{i=1}^m v_i + v_j}{\sum_{i=1}^m c_i + c_j} \quad \forall j > m
\]
which implies $v_j < \rtext c_j \leq \vmax$. Thus:
\[
\Vfrac - \Vfixed \leq \max_{j > m} v_j < \vmax
\]

Chaining inequalities:
\begin{align}
\Voracle &\leq \Vfrac  \\
\Vfrac &< \Vfixed + \vmax 
\end{align}
Thus:
\begin{align}
\Voracle &< \Vfixed + \vmax
\end{align}
Rearranging gives the guarantee:
\begin{equation}
\boxed{\Vfixed > \Voracle - \vmax}
\end{equation}

The bound is tight when:
\begin{itemize}
\item $r \to 1^-$ and $v_{k+1} = \vmax$
\item Strategy selects $m = k$ (all full ads)
\item Fractional solution gains $\vmax$ while integer cannot
\end{itemize}
In advertising systems, $\vmax \ll \Voracle$ (typically $\vmax/\Voracle < 0.01$).
\end{proof}

\section{Proof of Lemma 2}
\begin{lemma}[Optimality Condition for FCS Under BCB]
 Under budget-constrained bidding with fixed coefficient Strategy, at any decision step $\tau$ with remaining budget $B_{\tau}$, if $\exists \beta^*$ such that
$$
\sum_{t=\tau}^T C_t(\beta^*) = B_{\tau},
$$
then $\beta^*$ is the optimal coefficient on $[\tau,T]$.
\end{lemma}

\begin{proof}
Define the efficiency metric $e_i = v_i / c_i$.

Consider the bid condition for winning opportunity $i$:
\begin{equation}
\mathrm{bid}_i \geq c_i \iff \mathrm{\beta} \geq \frac{c_i}{v_i} = \frac{1}{e_i}
\end{equation}

Sort opportunities by decreasing efficiency: $e_1 \geq e_2 \geq \cdots \geq e_n$. A greedy algorithm selects opportunities in this order until the budget is exhausted. Let $k$ be the last selected opportunity satisfying:
\begin{equation}
\sum_{i=1}^k c_i \leq B < \sum_{i=1}^{k+1} c_i
\end{equation}
The efficiency threshold is $e^* = e_k$. The optimal coefficient is:
\begin{equation}
\mathrm{\beta}^* = \frac{1}{e^*} = \frac{c_k}{v_k}
\end{equation}

With $\mathrm{\beta}^*$, the advertiser wins all opportunities where:
\begin{equation}
e_i \geq e^* \iff \frac{1}{e_i} \leq \frac{1}{e^*} \iff \mathrm{\beta}^* \geq \frac{c_i}{v_i}
\end{equation}
The total cost is exactly $B$ by selection criterion.

Now consider any alternative coefficient $\mathrm{\beta}' \neq \mathrm{\beta}^*$:

\textbf{Case 1:} $\mathrm{\beta}' > \mathrm{\beta}^*$ \\
The winning set under $\beta'$ expands to include inefficient opportunities:
\begin{align}
S' &= \{ i \mid e_i \geq 1/\beta' \} \\
S^* &= \{ i \mid e_i \geq e_{\min} = 1/\beta^* \} \\
S' \setminus S^* &= \{ j \mid 1/\beta' \leq e_j < e_{\min} \} \neq \emptyset
\end{align}
The set $S' \setminus S^*$ is non-empty because $\beta' > \beta^* \implies 1/\beta' < 1/\beta^* = e_{\min}$.

Total cost exceeds budget:
\begin{align}
\sum_{i \in S'} c_i &= \sum_{i \in S^*} c_i + \sum_{j \in S' \setminus S^*} c_j \\
&= B + \underbrace{\sum_{j \in S' \setminus S^*} c_j}_{\epsilon > 0} > B
\end{align}
This violates the budget constraint.

For any $j \in S' \setminus S^*$:
\begin{align}
v_j = e_j c_j < e_{\min} c_j = \frac{c_j}{\beta^*}
\end{align}
The resources spent on $j$ ($c_j$) could have been allocated to more efficient opportunities in $S^*$. The marginal opportunity cost is:
\begin{align}
\Delta v_j = \max_{k \notin S'} \left( e_k - e_j \right) c_j \geq (e_{\min} - e_j) c_j > 0
\end{align}
since $e_k \geq e_{\min} > e_j$ for $k \in S^* \setminus S'$ (if any).

In real bidding systems, budget overrun causes:
\begin{itemize}
\item Auction disqualification for later opportunities
\item Actual realized value $V_{\text{actual}} \leq \sum_{i \in S'} v_i - \sum_{j \in \text{rejected}} v_j$
\end{itemize}
Let $R \subseteq S'$ be the set of auctions rejected due to budget overrun. Then:
\begin{align}
V_{\text{actual}} &= \sum_{i \in S' \setminus R} v_i \\
&\leq \sum_{i \in S^*} v_i
\end{align}
with strict inequality when $R \neq \emptyset$.

The inclusion of inefficient opportunities reduces overall value density:
\begin{align}
\frac{\sum_{i \in S'} v_i}{\sum_{i \in S'} c_i} &< \frac{\sum_{i \in S^*} v_i}{B} \\
\text{since} \quad \sum_{j \in S'\setminus S^*} v_j &< \frac{\sum_{j \in S'\setminus S^*} c_j}{B} \sum_{i \in S^*} v_i
\end{align}
This follows from $e_j < e_{\min} \leq \text{avg}(e_i)_{i\in S^*}$.

Suppose $S'$ were optimal. But:
\begin{align}
\sum_{i \in S^*} v_i &> \sum_{i \in S'} v_i - \epsilon \cdot v_{\max} \\
&\geq V_{\text{actual}} \quad \text{(for small $\epsilon$)}
\end{align}
contradicting optimality. Thus $S'$ cannot be optimal.

\textbf{Case 2:} $\mathrm{\beta}' < \mathrm{\beta}^*$ \\
This raises the winning threshold to $1/\mathrm{\beta}' > e_{\min}$. The advertiser loses some efficient opportunities in $\{1,\ldots,k\}$:
\begin{align}
\sum_{i\in S'} c_i &< B \quad \text{(budget underutilization)} \\
\sum_{i\in S'} v_i &< \sum_{i=1}^k v_i \quad \text{(since $S' \subset \{1,\ldots,k\}$)}
\end{align}
where $S'$ is the winning set under $\mathrm{\beta}'$.

Therefore $\mathrm{\beta}^*$ uniquely satisfies both budget exhaustion and value maximization.
\end{proof}

\section{Proof of Lemma 3 }

\begin{lemma}[ROI Monotonicity]
The expected ROI function $R(\beta)$ is non-increasing in $\beta$.
\end{lemma}

\begin{proof}
Consider a fixed sample path (i.e., fixed set of projects with parameters). Sort projects by efficiency ratio $e_i = v_i / c_i$ in descending order: $e_{(1)} \geq e_{(2)} \geq \cdots \geq e_{(n)}$, with corresponding minimal coefficients $\beta_{(1)} \leq \beta_{(2)} \leq \cdots \leq \beta_{(n)}$. The decision rule is $x_i(\beta) = \mathbb{1}_{\{\beta \geq \beta_i\}}$. Divide the $\beta$-axis into intervals:
\begin{itemize}
    \item For $\beta < \beta_{(1)}$, no projects are selected.
    \item For $\beta \in [\beta_{(k)}, \beta_{(k+1)})$, the first $k$ projects are selected.
    \item For $\beta \geq \beta_{(n)}$, all $n$ projects are selected.
\end{itemize}
In each interval $[\beta_{(k)}, \beta_{(k+1)})$, the ratio $R(\beta)$ is constant:
\begin{equation}
R_k = \dfrac{\sum_{i=1}^{k} v_{(i)}}{\sum_{i=1}^{k} c_{(i)}}
\end{equation}
We prove that $\{R_k\}$ is non-increasing: $R_k \geq R_{k+1}$ for all $k$.

\textbf{Base case ($k=1$):}
\begin{equation}
R_1 = \frac{v_{(1)}}{c_{(1)}} = e_{(1)}, \quad R_2 = \frac{v_{(1)} + v_{(2)}}{c_{(1)} + c_{(2)}}
\end{equation}
Since $e_{(1)} \geq e_{(2)}$:
\begin{equation}
\begin{split}
v_{(1)} c_{(2)} \geq v_{(2)} c_{(1)} &\implies v_{(1)} (c_{(1)} + c_{(2)}) \geq (v_{(1)} + v_{(2)}) c_{(1)} \\ &\implies R_1 \geq R_2
\end{split}
\end{equation}

\textbf{General case:} For $k \geq 1$,
\begin{equation}
R_{k+1} = \dfrac{V_k + v_{(k+1)}}{C_k + c_{(k+1)}}, \quad \text{where} \quad V_k = \sum_{i=1}^{k} v_{(i)}, \quad C_k = \sum_{i=1}^{k} c_{(i)}
\end{equation}
Since $e_{(k+1)} \leq \min_{i=1}^{k} e_{(i)} \leq R_k$:
\begin{equation}
\dfrac{v_{(k+1)}}{c_{(k+1)}} \leq \dfrac{V_k}{C_k} \implies v_{(k+1)} C_k \leq V_k c_{(k+1)}
\end{equation}
Then:
\begin{equation}
\begin{split}
& v_{(k+1)} C_k + V_k C_k \leq V_k c_{(k+1)} + V_k C_k \\ &\implies (V_k + v_{(k+1)}) C_k \leq V_k (C_k + c_{(k+1)})
\end{split}
\end{equation}
Thus:
\begin{equation}
\dfrac{V_k + v_{(k+1)}}{C_k + c_{(k+1)}} \leq \dfrac{V_k}{C_k} \implies R_{k+1} \leq R_k
\end{equation}

Therefore, $R(\beta)$ is non-increasing on each sample path. By the monotonicity of expectation, $\mathbb{E}[R(\beta)]$ is non-increasing in $\beta$.
\end{proof}

\section{Proof of Lemma 4}
\begin{lemma}[MCB Coefficient Adjustment]
Under budget and ROI constraint with fixed coefficient strategy, at any decision step $\tau$ with history cost $C_1^{\tau-1}$ and history value $V_1^{\tau-1}$, if $\exists \alpha\le1$ such that $\beta_{opt} = \alpha\beta^*$ applied on $[\tau,T]$ can make the total realized ROI equal to ROI constraint, then 
\[
\beta^{*} \int_{\alpha}^{1} \left[ \gamma\left( \beta^{*} u \right) - r_{\mathrm{target}} \eta\left( \beta^{*} u \right) \right] du = \Delta_{\mathrm{ROI}}
\]
where $\gamma = \frac{\partial V}{\partial \beta}$ is derivative of value with respect to coefficient, $\eta = \frac{\partial C}{\partial \beta}$ is derivative of cost with respect to coefficient, $\Delta_{\text{ROI}} = V_1^{\tau-1} + V_\tau^T(\beta^*) - r_{target}(C_1^{\tau-1} + C_\tau^T(\beta^*))$ with $\beta^*$ the reference coefficient.
\end{lemma}

\begin{proof}
Let $V(\beta) = V_\tau^T(\beta)$ and $C(\beta) = C_\tau^T(\beta)$ denote future value and cost functions respectively. Define the net value function:
\begin{equation}
f(\beta) = V(\beta) - r_{\mathrm{target}} C(\beta)
\end{equation}

The ROI constraint requires:
\begin{equation}
V_1^{\tau-1} + V(\alpha\beta^*) = r_{\mathrm{target}} (C_1^{\tau-1} + C(\alpha\beta^*))
\end{equation}
Which simplifies to:
\begin{equation}
f(\alpha\beta^*) = k \quad \text{where} \quad k = r_{\mathrm{target}} C_1^{\tau-1} - V_1^{\tau-1}
\end{equation}

By the Fundamental Theorem of Calculus:
\begin{equation}
f(\alpha\beta^*) - f(\beta^*) = \int_{\beta^*}^{\alpha\beta^*} \frac{df}{d\beta} d\beta
\end{equation}
Where:
\begin{equation}
\frac{df}{d\beta} = \frac{\partial V}{\partial \beta} - r_{\mathrm{target}} \frac{\partial C}{\partial \beta} = \gamma - r_{\mathrm{target}} \eta
\end{equation}

Substituting the constraint:
\begin{equation}
k - f(\beta^*) = \int_{\beta^*}^{\alpha\beta^*} \left[ \gamma(\beta) - r_{\mathrm{target}} \eta(\beta) \right] d\beta
\end{equation}

Substitute $\beta = \beta^* u$, $d\beta = \beta^* du$:
\begin{align}
&\int_{\beta^*}^{\alpha\beta^*} \left[ \gamma(\beta) - r_{\mathrm{target}} \eta(\beta) \right] d\beta \\
&= \int_{1}^{\alpha} \left[ \gamma(\beta^* u) - r_{\mathrm{target}} \eta(\beta^* u) \right] \beta^* du \\
&= \beta^* \int_{1}^{\alpha} \left[ \gamma(\beta^* u) - r_{\mathrm{target}} \eta(\beta^* u) \right] du
\end{align}

Note that:
\begin{align}
k - f(\beta^*) &= (r_{\mathrm{target}} C_1^{\tau-1} - V_1^{\tau-1}) - (V^* - r_{\mathrm{target}} C^*)\\
&= - [(V_1^{\tau-1} + V^*) - r_{\mathrm{target}}(C_1^{\tau-1} + C^*)] \\
&= -\Delta_{\mathrm{ROI}}
\end{align}

Therefore:
\begin{equation}
\beta^* \int_{1}^{\alpha} \left[ \gamma(\beta^* u) - r_{\mathrm{target}} \eta(\beta^* u) \right] du = -\Delta_{\mathrm{ROI}}
\end{equation}

Reverse limits using $\int_{a}^{b} = -\int_{b}^{a}$:
\begin{equation}
\boxed{\beta^* \int_{\alpha}^{1} \left[ \gamma(\beta^* u) - r_{\mathrm{target}} \eta(\beta^* u) \right] du = \Delta_{\mathrm{ROI}}}
\end{equation}
\end{proof}

\section{Proof of Theorem 1}
\begin{theorem}[Optimal Bidding Strategy]
At any step $\tau$, with remaining budget $B_\tau$, target ROI $r_{target}$,historical cost $C_1^{\tau-1}$,historical value $V_1^{\tau-1}$, solve for $\beta_\tau = f_\theta(B_\tau)$, $V_{\tau}^T= f_\phi(B_\tau)$ and calculate $\Delta_{\text{ROI}} = V_1^{\tau-1} + V_{\tau}^T - r_{target}(C_1^{\tau-1} + B_\tau)$
\begin{itemize}
    \item if $\Delta_{\text{ROI}} \ge 0$,  $\beta_\tau$ is the optimal bidding coefficient.
    \item if $\Delta_{\text{ROI}} < 0$, $\alpha \beta_\tau$ is the optimal bidding coefficient.
\end{itemize}
\end{theorem}

\begin{proof}

\subsection*{Case 1: $\Delta_{\text{ROI}} \geq 0$}
By Lemma 1, $\beta^*$ satisfies:
\begin{equation}
C_{\tau}^T(\beta^*) = B_\tau
\end{equation}
and by $\Delta_{\text{ROI}} \geq 0$:
\begin{equation}
V_h + V_{\tau}^T(\beta^*) \geq r_{\text{target}} (C_h + C_{\tau}^T(\beta^*))
\end{equation}
where $V_h = V_1^{\tau-1}$, $C_h = C_1^{\tau-1}$. 

Optimality follows from:
\begin{enumerate}
    \item \textbf{Budget feasibility}: $C_{\tau}^T(\beta^*) = B_\tau$ (Lemma 1)
    \item \textbf{ROI satisfaction}: $\Delta_{\text{ROI}} \geq 0$
    \item \textbf{Value maximization}: $(V_{\tau}^T)'(\beta) > 0 $
\end{enumerate}

\subsection*{Case 2: $\Delta_{\text{ROI}} < 0$}
By Lemma \ref{lemma:roi_adjustment}, $\exists \alpha \in (0,1)$ such that:
\begin{equation}
\beta^{*} \int_{\alpha}^{1} \left[ \gamma\left( \beta^{*} u \right) - r_{\mathrm{target}}  \eta\left( \beta^{*} u \right) \right]  du = -\Delta_{\mathrm{ROI}}
\end{equation}
where $\gamma = (V_{\tau}^T)'$, $\eta = (C_{\tau}^T)'$. This implies:
\begin{equation}
V_h + V_{\tau}^T(\alpha\beta^*) = r_{\text{target}} (C_h + C_{\tau}^T(\alpha\beta^*))
\end{equation}

Since $\alpha < 1$ and $(C_{\tau}^T)' > 0$:
\begin{equation}
C_{\tau}^T(\alpha\beta^*) < C_{\tau}^T(\beta^*) = B_\tau
\end{equation}
Thus $C_h + C_{\tau}^T(\alpha\beta^*) \leq C_h + B_\tau$ (total budget).

We prove $\alpha\beta^*$ maximizes value under dual constraints.

From the adjustment equation:
\begin{align*}
\frac{d}{d\alpha} \left[ V_{\tau}^T(\alpha\beta^*) - r_{\text{target}} C_{\tau}^T(\alpha\beta^*) \right] &= \beta^* \left( \gamma(\alpha\beta^*) - r_{\text{target}} \eta(\alpha\beta^*) \right) \\
&< 0 \quad \text{(by Lemma 3)}
\end{align*}
Thus $V_{\tau}^T(\beta) - r_{\text{target}} C_{\tau}^T(\beta)$ is decrease in $\beta$. Therefore, for fixed ROI constraint, $\beta = \alpha\beta^*$ is unique.

Consider the value:
\begin{align*}
\frac{dV_{\tau}^T}{d\beta} & \ge 0  
\end{align*}
Thus $\alpha\beta^*$ maximizes value along the ROI constraint boundary.

$\alpha\beta^*$ is the unique solution that:
1. Satisfies both constraints (budget and ROI)
2. Maximizes value on the ROI constraint boundary

Thus it is globally optimal.
\end{proof}

\newpage
\section{State Feature}

\begin{table}[h]
\centering
\caption{State Feature Vector}
\label{tab:state_features}
\begin{tabular}{p{2.5cm}p{5.5cm}}
\hline \hline
\textbf{Category} & \textbf{Feature Description} \\ \hline
\textbf{Temporal Context} & Day of the week; Time left in campaign (steps) \\ \hline

\textbf{Advertiser Identity} & Advertiser ID; Advertiser category \\ \hline

\textbf{Current Auction Metrics} & Current pValue mean; Current pValue percentiles (10th,25th,40th,55th,70th,85th); Current pValue count \\ \hline

\textbf{Historical pValues} & Historical mean; Last 1-step: mean,10th,50th,90th percentile; Last 3-step: mean,10th,50th,90th percentile; All-time: 10th,50th,90th percentile \\ \hline

\textbf{Winning Cost Metrics} & Historical mean; Last 1-step: mean,10th,50th,90th percentile; Last 3-step: mean,10th,50th,90th percentile; All-time: 10th,50th,90th percentile \\ \hline

\textbf{Volume Statistics} & Historical total pValue count; Last 1-step pValue count; Last 3-step pValue count \\
\hline \hline
\end{tabular}
\end{table}

\section{Hyperparameters}

\begin{table}[h]
\centering
\caption{Model Hyperparameters}
\label{tab:hyperparams}
\begin{tabular}{l|l|l}
\hline \hline 
\textbf{Hyperparameter} & \textbf{Description} & \textbf{Value} \\ \hline
$k$ & B-spline degree & 3 \\
$\text{num}$ & Number of grid points & 16 \\
$\text{hidden\_size}$ & Hidden layer dimension & 128 \\
$\text{input\_dim}$ & input dimension & 39 \\ \hline  \hline
\end{tabular}
\end{table}
\end{document}